\documentclass[english,11pt]{article}

\usepackage[utf8]{inputenc} 
\usepackage[T1]{fontenc}    
\usepackage{hyperref}       
\usepackage{url}            
\usepackage{booktabs}       
\usepackage{amsfonts}       
\usepackage{nicefrac}       
\usepackage{microtype}      

\usepackage{amsmath}
\usepackage{amsthm}
\usepackage{amssymb}
\usepackage{subcaption}
\makeatletter
\numberwithin{equation}{section}
\numberwithin{figure}{section}
\theoremstyle{plain}
\newtheorem{thm}{\protect\theoremname}
\theoremstyle{plain}
 
\newtheorem{corollary}[thm]{Corollary}
\newtheorem{theorem}{Theorem}[section]
\newtheorem{prop}[theorem]{Proposition}
\newtheorem{lemma}[theorem]{Lemma}

\theoremstyle{definition}
\newtheorem{definition}[thm]{Definition}
\newtheorem{remark}[thm]{Remark}

\usepackage{dsfont}

\newcommand{\bbE}{\ensuremath{\mathbb{E}}}
\newcommand{\bbN}{\ensuremath{\mathbb{N}}}
\newcommand{\bbR}{\ensuremath{\mathbb{R}}}
\newcommand{\bbZ}{\ensuremath{\mathbb{Z}}}
\newcommand{\mcC}{\ensuremath{\mathcal{C}}}
\newcommand{\mcF}{\ensuremath{\mathcal{F}}}
\newcommand{\mcN}{\ensuremath{\mathcal{N}}}

\newcommand{\eps}{\ensuremath{\varepsilon}}
\DeclareMathOperator*{\E}{\mathbb{E}}
\DeclareMathOperator{\Cov}{Cov}
\DeclareMathOperator{\Var}{Var}

\DeclareMathOperator{\eSDA}{\eps-SDA}
\DeclareMathOperator{\VSTAT}{VSTAT}
\DeclareMathOperator{\STAT}{STAT}

\DeclareMathOperator{\poly}{poly}

\usepackage{hyperref}
\hypersetup{
    colorlinks,
    allcolors=red
}
\usepackage[lined,boxed,ruled,norelsize,algo2e]{algorithm2e}

\usepackage{todonotes}
\usepackage{fullpage}

\makeatother

\usepackage{babel}
  \providecommand{\lemmaname}{Lemma}
\providecommand{\theoremname}{Theorem}

\begin{document}
\global\long\def\defeq{\stackrel{\mathrm{{\scriptscriptstyle def}}}{=}}
\global\long\def\norm#1{\left\Vert #1\right\Vert }
\global\long\def\R{\mathbb{R}}
 \global\long\def\Rn{\mathbb{R}^{n}}
\global\long\def\tr{\mathrm{Tr}}
\global\long\def\diag{\mathrm{diag}}
\global\long\def\Diag{\mathrm{Diag}}
\global\long\def\C{\mathbb{C}}
\global\long\def\vol{\text{vol}}
 \def\eps{\epsilon}
 \def\D{{\cal D}}
 \def\P{{\cal P}}

\title{On the Complexity of Learning Neural Networks}

\author{
  Le Song\\
  \texttt{lsong@cc.gatech.edu} \\
 \and
 Santosh Vempala\\
\texttt{vempala@gatech.edu} \\
 \and
 John Wilmes\\
 \texttt{wilmesj@gatech.edu} \\
\and
Bo Xie\\
\texttt{bo.xie@gatech.edu} \\
}

\date{Georgia Institute of Technology \\[2ex] \today}

\maketitle

\begin{abstract}
The stunning empirical successes of neural networks currently lack rigorous
theoretical explanation. What form would such an explanation take, in the face
of existing complexity-theoretic lower bounds? A first step might be to show that
data generated by neural networks with a single hidden layer, smooth activation functions and benign input distributions can be learned efficiently. We demonstrate here a 
comprehensive lower bound ruling out this possibility: for a wide class of
activation functions (including all currently used), and inputs drawn from any
logconcave distribution, there is a family of one-hidden-layer functions whose
output is a sum gate, that are hard to learn in a precise sense: any {\em
statistical query} algorithm (which includes all known variants of
stochastic gradient descent with any loss function) needs an exponential number
of queries even using tolerance inversely proportional to the input dimensionality.
Moreover, this hard family of functions is realizable with a small (sublinear
in dimension) number of activation units in the single hidden layer. The lower
bound is also robust to small perturbations of the true weights. Systematic
experiments illustrate a phase transition in the training error as predicted by
the analysis.
\end{abstract}

\section{Introduction}

It is well-known that Neural Networks (NN's) provide universal approximate
representations
\cite{hornik1989multilayer,cybenko1989approximation,barron1993universal} and
under mild assumptions, i.e., any real-valued function can be approximated by a
NN. This holds for a wide class of activation functions (hidden layer units)
and even with only a single hidden layer (although there is a trade-off between
depth and width \cite{eldan2016power, telgarsky2016benefits}).  Typically
learning a NN is done by stochastic gradient descent applied to a loss function
comparing the network's current output to the values of the given training
data; for regression, typically the function is just the least-squares error.
Variants of gradient descent include drop-out, regularization, perturbation,
batch gradient descent etc. In all cases, the training algorithm has the
following form:

\begin{centering}
\fbox{\parbox{0.95\textwidth}{
Repeat:
\begin{enumerate}
\item Compute a fixed function $F_W(.)$ defined by the current network weights $W$ on a subset of
training examples.
\item Use $F_W(.)$ to update the current weights $W$.
\end{enumerate}
}
}

\end{centering}

The empirical success of this approach raises the question: what can NN's learn
efficiently {\em in theory}? In spite of much effort, at the moment there are
no satisfactory answers to this question, even with reasonable assumptions on
the function being learned and the input distribution. 

When learning involves some computationally intractable optimization problem,
e.g., learning an intersection of halfspaces over the uniform distribution on
the Boolean hypercube, then any training algorithm is unlikely to be efficient. This is the case even for improper learning (when the complexity of the hypothesis class being
used to learn can be greater than the target class). Such lower
bounds are unsatisfactory to the extent they rely on
discrete (or at least nonsmooth) functions and distributions.  What if we assume that the
function to be learned is generated by a NN with a
single hidden layer of smooth activation units, and the input distribution is
benign? Can such functions be learned efficiently by gradient descent?  

Our main result is a lower bound, showing a simple and natural family of functions generated by 
$1$-hidden layer NN's using any known activation function (e.g., sigmoid,
ReLU), with each input drawn from a logconcave input distribution (e.g.,
Gaussian, uniform in an interval), are hard to learn by a wide class of
algorithms, including those in the general form above. Our finding implies that
efficient NN training algorithms need to use stronger assumptions on the target
function and input distribution, more so than Lipschitzness and smoothness even
when the true data is generated by a NN with a single hidden layer.

The idea of the lower bound has two parts. First, NN updates can be viewed as
{\em statistical queries} to the input distribution. Second, there are many
very different $1$-layer networks, and in order to learn the correct one, any
algorithm that makes only statistical queries of not too small accuracy has to
make an exponential number of queries. The lower bound uses the SQ framework of
Kearns~\cite{kearns1998efficient} as generalized by Feldman et
al.~\cite{feldman2013statistical}. 

\subsection{Statistical query algorithms}
A statistical query (SQ) algorithm is one that solves a computational problem
over an input distribution; its interaction with the input is limited
to querying the expected value of 
of a bounded function up to a desired accuracy. More precisely, for any integer $t >
0$ and distribution $D$ over $X$, a \textbf{$\VSTAT(t)$ oracle} takes as input a \textbf{query function} $f: X \to
[0,1]$ with expectation $p=\E_D(f(x))$ and returns a value $v$ such that
\[
    \left|p - v\right| \le \max \left\{\frac{1}{t},
        \sqrt{\frac{p(1-p)}{t}}\right\}.
\]
The bound on the RHS is the standard deviation of $t$ independent Bernoulli coins with desired expectation, i.e., the error that even a random sample of size $t$ would yield.
In this paper, we study SQ algorithms that access the input distribution only
via the $\VSTAT(t)$ oracle. The remaining computation is unrestricted and can use randomization (e.g., to determine which query to ask next).

In the case of an algorithm training a neural network via gradient descent, the
relevant query functions are derivatives of the loss function.

The statistical query framework was first introduced by Kearns for supervised
learning problems~\cite{Kearns93} using the $\STAT(\tau)$ oracle, which,
for $\tau \in \bbR_{+}$,
responds to a query function $f: X \to [0,1]$ with a value $v$ such that
$|\E_D(f) - v| \le \tau$. The $\STAT(\sqrt{\tau})$ oracle can be simulated by the
$\VSTAT(O(1/\tau))$ oracle. 
The $\VSTAT$ oracle was introduced by
\cite{feldman2013statistical} who extended these oracles to general problems
over distributions. 

\subsection{Main result}

We will describe a family $\mcC$ of functions $f:\bbR^n \rightarrow \bbR$ that can
be computed exactly by a small NN, but cannot be efficiently learned by an SQ
algorithm.  While our result applies to all commonly used activation units, we
will use sigmoids as a running example. Let $\sigma(z)$ be
the sigmoid gate that goes to $0$ for $z<0$ and goes to $1$ for $z>0$. The
sigmoid gates have sharpness parameter $s$:
\[
    \sigma(x) = \sigma_s(x) =\frac{e^{sx}}{1+e^{sx}}.
\]
Note that the parameter $s$ also bounds the Lipschitz constant of $\sigma(x)$. 

A function $f: \bbR^n \to \bbR$ can be computed exactly by a single layer NN
with sigmoid gates precisely when it is of the form $f(x) = h(\sigma(g(x))$,
where $g: \bbR^n \to \bbR^m$ and $h: \bbR^m \to \bbR$ are affine, and $\sigma$
acts component-wise. Here, $m$ is the number of hidden units, or sigmoid gates,
of the of the NN.

In the case of a learning problem for a class $\mcC$ of functions $f:X
\to \bbR$, the input distribution to the algorithm is over labeled examples
$(x, f^*(x))$, where $x \sim D$ for some underlying distribution $D$ on $X$,
and $f^* \in \mcC$ is a fixed concept (function). 

As mentioned in the introduction, we can view a NN learning algorithm
as a statistical query (SQ) algorithm: in each iteration, the algorithm constructs a function
based on its current weights (typically a gradient or subgradient),
evaluates it on a \emph{batch} of random examples from the input
distribution, then uses the evaluations to update the weights of the
NN. Then we have the following result.  



\begin{theorem}
\label{thm:main} Let $n \in \bbN$, and let $\lambda, s \ge 1$. 
There exists an explicit family $\mcC$ of functions $f:\bbR^n
\to [-1,1]$, representable as a single hidden layer neural network with 
    $O(s\sqrt{n}\log(\lambda sn))$ sigmoid units of sharpness $s$, a single output
    sum gate and a weight matrix with condition number
    $O(\poly(n,s,\lambda))$,  and an integer $t = \Omega(s^2 n)$ s.t. the following holds.
    Any (randomized) SQ algorithm $A$ that uses $\lambda$-Lipschitz queries to $\VSTAT(t)$ and weakly learns $\mcC$ 
    with probability at least $1/2$, to within regression error
    $1/\sqrt{t}$ less than any constant function
    over i.i.d.\ inputs from any logconcave distribution of
    unit variance on $\bbR$ requires
    $2^{\Omega(n)}/(\lambda s^2)$ queries. 
\end{theorem}
The Lipschitz assumption on the statistical queries is satisfied by all commonly used algorithms for training neural networks can
be simulated with Lipschitz queries (e.g., gradients of natural loss
functions with regularizers). This assumption can be omitted if the output of the hard-to-learn family $\mcC$ is represented with
finite precision (see Corollary~\ref{cor:finite-precision}).

Informally, Theorem~\ref{thm:main} shows that there exist simple realizable functions that are not
efficiently learnable by NN training algorithms with polynomial batch
sizes, assuming the algorithm allows for error as much as the standard deviation of
random samples for each query.  We remark that in
practice, large batch sizes are seldom used for training NNs, not just for efficiency, but also since moderately
noisy gradient estimates are believed to be useful for avoiding bad local
minima. 
Even NN training algorithms with larger batch sizes will require
$\Omega(t)$ samples to achieve lower error, whereas the NNs that represent
functions in our class $\mcC$ have only $\widetilde{O}(\sqrt{t})$ parameters.

Our lower bound extends to a broad family of
activation units, including all the well-known ones (ReLU, sigmoid, softplus
etc., see Section~\ref{sec:other-activation}). In the case of sigmoid gates, the
functions of $\mcC$ take the following form (cf.\
Figure~\ref{fig:phi}). For a set $S \subseteq \{1,\ldots, n\}$, we define 
$f_{m,S}(x_1,\ldots, x_n) = \phi_m(\sum_{i \in S} x_i)$, where
\begin{equation}\label{fmS}
    \phi_m(x) = -(2m+1) + \sum_{k = -m}^m \sigma\left(
    x - \frac{(4k-1)}{s} \right) + \sigma\left(\frac{(4k+1)}{s} - x\right)\,.
\end{equation}
Then $\mcC = \{ f_{m,S} : S \subseteq \{1,\ldots, n\}\}$. We call the functions
$f_{m,S}$, along with $\phi_m$, the \emph{$s$-wave} functions. It is easy to
see that they are smooth and bounded. Furthermore, the size of the NN
representing this hard-to-learn
family of functions is only $\tilde{O}(s\sqrt{n})$, assuming the query
functions (e.g., gradients of loss function) are $\poly(s,n)$-Lipschitz. 
We note that the lower bounds hold regardless of the architecture of the model,
i.e., NN used to learn.

As we show empirically in Section~\ref{sec:expts}, these lower bounds hold even
for small values of $n$ and $s$, across choices of gates, architecture used to
learn, learning rate, batch size, etc. As suggested by the statement of
Theorem~\ref{thm:main}, there is threshold for the quantity $s\sqrt{n}$, above
which stochastic gradient descent fails to train the NN to low error --- the regression error of the trained NN does not even improve on
that of a constant function.

The condition number upper bound for $\mcC$ is significant in part because
there do exist SQ algorithms for learning certain families of simple NNs with
time complexity polynomial in the condition number of the weight matrix (the tensor factorization based algorithm of Janzamin et al. \cite{Janzamin15} can easily be seen to be SQ). Our
results imply that this dependence cannot be substantially improved (see
Section~\ref{sec:related}).
\begin{figure}
    \centering
    \begin{subfigure}[t]{0.4\textwidth}
        \includegraphics[width=\textwidth]{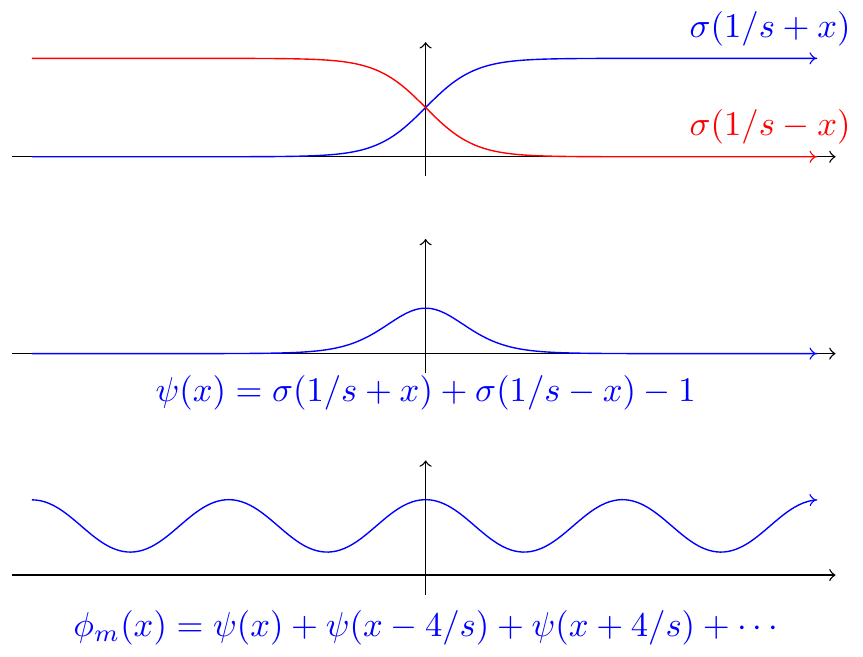}
    \caption{The sigmoid function, the $L^1$-function $\psi$ constructed from
        sigmoid functions, and the nearly-periodic ``wave'' function $\phi$
        constructed from $\psi$.}
    \end{subfigure}%
    \hspace{0.1\textwidth}
    \begin{subfigure}[t]{0.4\textwidth}
        \includegraphics[width=\textwidth]{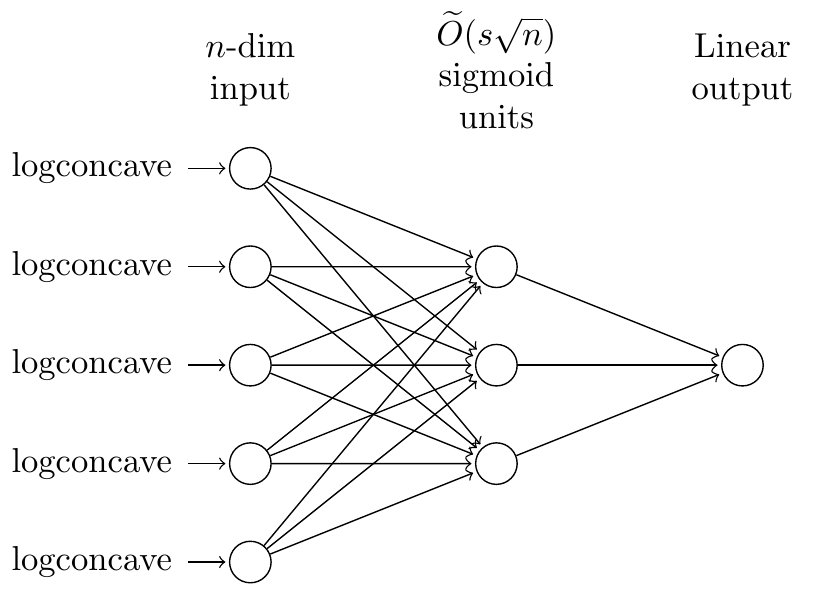}
        \caption{The architecture of the NNs computing the wave functions.}
    \end{subfigure}
    \caption{The wave function and its NN.}\label{fig:phi}
\end{figure}


\begin{remark}
    The class of input distributions can be relaxed further. Rather than
    being a product distribution, it suffices if the distribution is in
    isotropic position and invariant under reflections across and permutations
    of coordinate axes.  And instead of being logconcave, it suffices for
    marginals to be unimodal with variance $\sigma$, density $O(1/\sigma)$ at
    the mode, and density $\Omega(1/\sigma)$ within a standard deviation of the
    mode.
\end{remark}

Overall, our lower bounds suggest that even the combination of small network
size, smooth, standard activation functions, and benign input distributions is
insufficient to make learning a NN easy, even improperly via a very general
family of algorithms. Instead, stronger structural assumptions on the NN, such
as a small condition number, and very strong
structural properties on the input distribution, are necessary to make learning
tractable. It is our hope that these insights will guide the discovery of provable efficiency guarantees.

\subsection{Related Work}\label{sec:related}

There is much work on complexity-theoretic hardness of learning neural
networks~\cite{BR92, Daniely16, Klivans16}. These results have shown
the hardness of learning functions representable as small (depth $2$) neural
networks over discrete input distributions. Since these input distributions
bear little resemblance to the real-world data sets on which NNs have seen
great recent empirical success, it is natural to wonder whether more realistic
distributional assumptions might make learning NNs tractable. Our results
suggest that benign input distributions are insufficient, even for
functions realized as small networks with standard, smooth activation units.

Recent independent work of Shamir~\cite{Shamir16} shows a smooth family of
functions for which the gradient of the squared loss function is not
informative for training a NN over a Gaussian input distribution (more
generally, for distributions with rapidly decaying Fourier coefficients). In
fact, for this setting the paper shows an exponentially small bound on the
gradient, relying on the fine structure of the Gaussian distribution and of the
smooth functions (see~\cite{Shalev17} for a follow-up with experiments and
further ideas). These smooth functions cannot be realized in small NNs using
the most commonly studied activation units (though a related non-smooth family
of functions for which the bounds apply can be realized by larger NNs using
ReLU units). In contrast our bounds are (a) in the more general SQ framework,
and in particular apply regardless of the loss function, regularization scheme,
or specific variant of gradient descent (b) apply to functions actually
realized as small NNs using any of a wide family of activation units (c) apply
to any logconcave input distribution and (d) are robust to small perturbations
of the input layer weights. 

Also related is the tensor-based algorithm of Janzamin et al.~\cite{Janzamin15}
to learn a 1-layer network under nondegeneracy assumptions on the weight
matrix. The complexity is polynomial in the dimension, size of network being
learned and condition number of the weight matrix. Since their tensor
decomposition can also be implemented as a statistical query algorithm, our
results give a lower bound indicating that such a polynomial dependence on the
dimension and condition number is unavoidable. 

Other algorithmic results for learning NNs apply in very restricted settings.
For example, polynomial-time bounds are known 
for learning NNs with a single hidden ReLU layer over Gaussian
inputs under
the assumption that the hidden units use disjoint sets of inputs~\cite{bg17},
as well as for learning a single ReLU~\cite{GKKT16} and for learning sparse polynomials via
NNs~\cite{apvz14}.

\subsection{Proof ideas}

Our starting point is the observation is that NN training algorithms are
inherently statistical and can be simulated by $\VSTAT$.
This is because in each step a gradient is computed by
averaging over a batch of random examples. The number of samples needed in the
average depends only on the range of the functions being queried. In addition
to these queries, the algorithm is allowed to perform any other
computations that do not query the input. 

One way to prove a lower bound on the number of queries is to exhibit neural
networks approximating parity functions.  Since parity of an unknown subset is
hard-to-learn for SQ algorithms, this would give a lower bound in the setting
of neural networks as well.  However, the resulting bound on tolerance is much
worse since the discrete parity function would have to be approximated by a
continuous function.

Instead, we directly estimate the {\em statistical dimension} of the family of
$s$-wave functions. Statistical dimension is a
key concept in the study of SQ algorithms, and is known to
characterize the query complexity of supervised learning via SQ
algorithms~\cite{Blum+94,Szorenyi:09,feldman2013statistical}. Briefly, a family $\mcC$ of distributions (e.g., over
labeled examples) has statistical dimension $d$ with average correlation
$\bar{\gamma}$ if every $(1/d)$-fraction of $\mcC$ has average correlation
$\bar{\gamma}$; this condition implies that $\mcC$ cannot be learned with fewer
than $O(d)$ queries to $\VSTAT(O(1/\bar{\gamma}))$. See
Section~\ref{sec:statdim} for precise statements.

The SQ literature for supervised learning of boolean functions is rich.
However, lower bounds for regression problems in the SQ framework have so far
not appeared in the literature, and the existing notions of statistical
dimension are too weak for this setting. We state a new, strengthened notion of
statistical dimension for regression problems (Definition~\ref{def:eps-sda}), and
show that lower bounds for this dimension transfer to query complexity bounds
(Theorem~\ref{thm:sda}). The essential difference from the 
statistical dimension for learning is that we must additionally bound the average
covariances of \emph{indicator functions} (or, rather, continuous analogues of
indicators) on the outputs of functions in $\mcC$.
The essential claim in our lower bounds is therefore in showing that a typical
pair of (indicator functions on outputs of) $s$-wave functions has small
covariance.

Let $f: \bbR \to \bbR$ and let $S, T \subseteq \{1,\ldots, n\}$. Suppose
$S\setminus T$ and $T \setminus S$ are both large, and $D$ is a product
distribution on $\bbR^n$. Then
\begin{align*}
    &\E_{(x_1,\ldots, x_n) \sim D}(f(\sum_{i \in S} x_i)f(\sum_{i \in T} x_i)
    \mid \sum_{i \in S\cap T} x_i = z)  \\
    &=
    \E_{x_i, i \in S\setminus T}(f(\sum_{i \in
    S\setminus T} x_i + z))\E_{x_i, i \in T\setminus S}(f(\sum_{i \in
    T\setminus S} x_i + z))\,.
\end{align*}
So it suffices to show that the expectation of $f(\sum_{i \in S} x_i)$ doesn't
change much when we condition on the value of $z = \sum_{i \in S\cap T} x_i$.

We make the following observation: suppose $f$, like an
indicator function composed with an $s$-wave functions, is ``close to'' a
periodic function with period $\theta > 0$ (see Section~\ref{sec:quasiperiodic} for
a precise statement). Then for any logconcave distribution $D'$ on $\bbR$ of
variance $\sigma > \theta$, and any translation $z \in \bbR$, we have
\[
    \left|\E_{x \sim D}(f(x + z) - f(x))\right| =
    O\left(\frac{\theta}{\sigma}\right)\E_{x \sim D}(|f(x)|)
\]

In particular, conditioning on the value of $z = \sum_{i \in S\cap T} x_i$ has
little effect on the value of $f(\sum_{i \in S} x_i)$. The combination of these
observations gives the query complexity lower bound. Precise statements of some
of the technical lemmas are given in Section~\ref{sec:periodic}, and the complete
proof appears in Section~\ref{sec:proofs}.

\section{The complexity of learning smooth one-layer networks}

\subsection{Statistical dimension}\label{sec:statdim}
We now give a precise definition of the \emph{statistical dimension with
average correlation} for regression problems, extending the concept introduced 
in~\cite{feldman2013statistical}.

Let $\mcC$ be a finite family of functions $f: X \to \bbR$ over some domain
$X$, and let $D$ be a distribution over $X$.
The average covariance and the average correlation of $\mcC$ with respect to
$D$ are
\[
    \Cov_D(\mcC) = \frac{1}{|\mcC|^2} \sum_{f,g \in \mcC} \Cov_D(f,g) \quad \textrm{
        and } \quad 
    \rho_D(\mcC) = \frac{1}{|\mcC|^2}\sum_{f,g \in \mcC} \rho_D(f,g)
\]
where $\rho_D(f,g) = \Cov_D(f,g)/\sqrt{\Var(f)\Var(g)}$ when both $\Var(f)$ and
$\Var(g)$ are nonzero, and $\rho_D(f,g) = 0$ otherwise.


 

For $y \in \bbR$ and $\eps > 0$, we define the \textbf{$\eps$-soft indicator function}
$\chi_y^{(\eps)}: \bbR \to \bbR$ as
\[
\chi_y^{(\eps)}(x) = \chi_y(x) = \max\{0,1/\eps - (1/\eps)^2|x-y|\}. 
\]
So $\chi_y$ is $(1/\eps)^2$-Lipschitz, is
supported on $(y-\eps, y+\eps)$, and has norm $\|\chi_y\|_1 = 1$.

\begin{definition}\label{def:eps-sda}
    Let $\bar{\gamma} > 0$, let $D$ be a probability distribution over some
    domain $X$, and let $\mcC$ be a family of functions $f:X \to [-1,1]$ that
    are identically distributed as random variables over $D$.
    The \textbf{statistical dimension} of $\mcC$ relative to $D$ with average
    covariance $\bar{\gamma}$ and precision $\eps$, denoted by $\eSDA(\mcC, D,
    \bar{\gamma})$, is defined to be the largest integer $d$ such that
    the following holds:
    for every $y \in \bbR$ and every subset $\mcC' \subseteq \mcC$ of size
    $|\mcC'| > |\mcC|/d$, we have $\rho_D(\mcC') \le \bar{\gamma}$.
    Moreover, $\Cov_D(\mcC_y') \le (\max\{\eps, \mu(y)\})^2\bar{\gamma}$ where $\mcC_y' = \{\chi_y^{(\eps)} \circ f : f \in \mcC\}$ and
    $\mu(y) = E_D(\chi_y^{(\eps)} \circ f)$ for some $f \in
    \mcC$. 
\end{definition}
Note that the parameter $\mu(y)$ is independent of the choice of $f \in \mcC$. The application of this notion of dimension is given by the following theorem.



\begin{theorem}\label{thm:sda}
    Let $D$ be a distribution on a domain $X$ and let $\mcC$ be a family of
    functions $f: X \to [-1,1]$ identically distributed as random variables
    over $D$. Suppose there is $d \in \bbR$ and $\lambda \ge 1 \ge \bar{\gamma}
    > 0$ such that  $\eSDA(\mcC, D, \bar{\gamma}) \ge d$, where $\eps \le
    \bar{\gamma}/(2\lambda)$.
    Let $A$ be a randomized algorithm learning $\mcC$ over $D$ with probability
    greater than $1/2$ to regression error less than $\Omega(1) -
    2\sqrt{\bar{\gamma}}$. If $A$ only
    uses queries to $\VSTAT(t)$ for some $t = O(1/\bar{\gamma})$, which are
    $\lambda$-Lipschitz at any fixed $x \in X$, then $A$ uses $\Omega(d)$
    queries.
\end{theorem}

A version of the theorem for Boolean functions is proved
in~\cite{feldman2013statistical}.
For completeness, we include a proof of the version used in this
paper, following ideas in~\cite[Theorem 2]{Szorenyi:09}. The proof is given in
Section~\ref{sec:sda-proof}.

As a consequence of Theorem~\ref{thm:sda}, there is no need to consider an SQ
algorithm's query strategy in order to obtain lower bounds on its query
complexity. Instead, the lower bounds follow directly from properties of the
concept class itself, in particular from bounds on average covariances of
indicator functions. Theorem~\ref{thm:main} will therefore follow from
Theorem~\ref{thm:sda} by analyzing the statistical dimension of the $s$-wave
functions.

\section{Statistical dimension of one-layer functions}\label{sec:periodic}

We now present the most general context in which we obtain SQ lower bounds.

A function $\phi: \bbR \to \bbR$ is \textbf{$(M,\delta,\theta)$-quasiperiodic} if
there exists a function $\tilde{\phi}: \bbR \to \bbR$ which is periodic with
period $\theta$ such that $|\phi(x)-\tilde{\phi}(x)| < \delta$ for all $x \in
[-M,M]$. In particular, any periodic function with period $\theta$ is
$(M,\delta,\theta)$-quasiperiodic for all $M, \delta > 0$.

\begin{lemma}\label{lem:periodic-sda}
    Let $n \in \bbN$ and let $\theta > 0$. There exists $\bar{\gamma} =
    O(\theta^2/n)$ such that for all $\eps > 0$, there exist $M =
    O(\sqrt{n}\log(n/(\eps\theta))$ and $\delta = \Omega(\eps^3\theta/\sqrt{n})$ and a family
    $\mcC_0$ of affine functions $g: \bbR^n \to \bbR$ of bounded operator norm
    with the following property.
    Suppose $\phi: \bbR \to [-1,1]$ is
    $(M,\delta,\theta)$-quasiperiodic and
    $\Var_{x \sim U(0,\theta)}(\phi(x)) = \Omega(1)$.
    Let $D$ be logconcave distribution with unit variance on $\bbR$. 
    Then for $\mcC = \{\phi \circ g : g \in \mcC_0\}$, we have
    $\eSDA(\mcC, D^n, \bar{\gamma}) \ge 2^{\Omega(n)}\eps\theta^2$.
    Furthermore, the functions of $\mcC$ are
    identically distributed as random variables over $D^n$.
\end{lemma}

In other words, we have statistic dimension bounds (and hence query complexity
bounds) for functions that are sufficiently close to periodic.
However, the activation units of interest are generally monotonic increasing functions
such as sigmoids and ReLUs that are quite far from periodic. Hence, in order to
apply Lemma~\ref{lem:periodic-sda} in our context, we must show that the
activation units of interest can be combined to make nearly periodic functions.

In order to state and prove our results in
a general framework, we analyze as an intermediate step functions in
$L^1(\bbR)$, i.e., functions whose absolute value has bounded integral over the
whole real line. These $L^1$-functions analyzed in our framework are
themselves constructed as affine combinations of the usual activation
functions. For example, for the sigmoid unit with sharpness $s$, we study the
function
\[
\psi(x) = \sigma\left(\frac{1}{s}+x\right) + \sigma\left(\frac{1}{s} - x\right) - 1.
\]
The definition of our hard-to-learn family $f_{m,S}$ is exactly in this form (\ref{fmS}).

We now describe the properties of the integrable functions $\psi$ that will be used in the proof.

\begin{definition}
    For $\psi \in L^1(\bbR)$, we say the \textbf{essential radius} of $\psi$ is
    the number $r \in \bbR$ such that $\int_{-r}^r|\psi| = (5/6)\|\psi\|_1$.
\end{definition}

\begin{definition}\label{def:mbp}
    We say $\psi \in L^1(\bbR)$ has the \textbf{mean bound property} if for all
    $x \in \bbR$ and $\eps > 0$, we have 
    \[
        \psi(x) = O\left(\frac{1}{\eps}\int_{x-\eps}^{x+\eps} |\psi(x)|\right).
    \]
\end{definition}

In particular, if $\psi$ is bounded, and monotonic nonincreasing (resp.\
nondecreasing) for sufficiently large positive (resp.\ negative) inputs, then
$\psi$ satisfies Definition~\ref{def:mbp}. Alternatively, it suffices for
$\psi$ to have bounded first derivative.


To complete the proof of Theorem~\ref{thm:main}, we show that we can combine
activation units $\psi$ satisfying the above properties in a function which is
close to periodic, i.e., which satisfies the hypotheses of
Lemma~\ref{lem:periodic-sda} above.

\begin{lemma}\label{lem:l1-quasiperiodic}
    Let $\psi \in L^1(\bbR)$ have the mean bound property and let $r > 0$ be
    such that $\psi$ has essential radius at most $r$ and $\|\psi\|_1 =
    \Theta(r)$. Let $M,\delta > 0$. Then there is a pair of affine functions $h: \bbR^m
    \to \bbR$ and $g: \bbR \to \bbR^m$ such that if $\phi(x) = h(\psi(g(x)))$,
    where $\psi$ is applied component-wise, then $\phi$ is $(M,\delta,
    4r)$-quasiperiodic. Furthermore, $\phi(x) \in [-1,1]$ for all $x \in \bbR$,
    and $\Var_{x \sim U(0,4r)}(\phi(x)) = \Omega(1)$,
    and we may take $m = (1/r)\cdot O(\max\{m_1, M\})$, where $m_1$ satisfies
    \[
        \int_{m_1}^\infty (|\psi(x)| + |\psi(-x)|)dx < 4\delta r\,.
    \]
\end{lemma}

We now sketch how Lemmas~\ref{lem:periodic-sda} and~\ref{lem:l1-quasiperiodic}
imply Theorem~\ref{thm:main} for sigmoid units. 

\begin{proof}[Sketch of proof of Theorem~\ref{thm:main}]
    The sigmoid function $\sigma$ with sharpness $s$ is not even in
    $L^1(\bbR)$, so it is unsuitable as the function $\psi$ of
    Lemma~\ref{lem:l1-quasiperiodic}. Instead, we define $\psi$ to be an affine
    combination of $\sigma$ gates, namely
    \[
        \psi(x) = \sigma\left(\frac{1}{s} + x\right) + \sigma\left(\frac{1}{s}
        - x\right) - 1.
    \]
    Then $\psi$ satisfies the hypotheses of Lemma~\ref{lem:l1-quasiperiodic}.

    Let $\theta = 4r$ and let $\bar{\gamma} = O(\theta^2/n)$ be as given by the statement of
    Lemma~\ref{lem:periodic-sda}. Let $\eps = \bar{\gamma}/(2\lambda)$, and let 
    $M = O(\sqrt{n}\log(n/(\eps\theta))$ and $\delta = \Omega(\eps^3\theta/\sqrt{n})$
    be as given by the statement of Lemma~\ref{lem:periodic-sda}. By
    Lemma~\ref{lem:l1-quasiperiodic}, there is
    $m \in \bbN$ and 
    functions $h: \bbR^m \to \bbR$ and $g: \bbR \to \bbR^m$ such that
    $\phi = h \circ \psi \circ g$ is $(M,\delta,\theta)$-quasiperiodic and
    satisfies the hypotheses of Lemma~\ref{lem:periodic-sda}.
    Therefore, we have a family $\mcC_0$ of affine functions $f: \bbR^n \to \bbR$ such
    that for $\mcC = \{\phi \circ f : f \in \mcC_0\}$ satisfies
    $\eSDA(\mcC, D, \bar{\gamma}) \ge 2^{\Omega(n)}\eps\theta^2$.
    Therefore, the functions in $\mcC$ satisfy the hypothesis of
    Theorem~\ref{thm:sda}, giving the query complexity lower bound.
    
    The details are deferred to Section~\ref{sec:proofs}.
\end{proof}

\subsection{Different activation functions}\label{sec:other-activation}

Similar proofs give corresponding lower bounds for activation functions other
than sigmoids. In every case, we reduce to gates satisfying the hypotheses of
Lemma~\ref{lem:l1-quasiperiodic} by constructing an appropriate $L^1$-function
$\psi$ as an affine combination of
of the activation functions.

For example, let $\sigma(x) = \sigma_s(x) = \max\{0, sx\}$ denote the ReLU unit
with slope $s$. Then the affine combination
\begin{equation}\label{eq:relu}
    \psi(x) = \sigma(x + 1/s) - \sigma(x) + \sigma(-x + 1/s) - \sigma(-x) - 1
\end{equation}
is in $L^1(\bbR)$, and is zero for $|x| \ge 1/s$ (and hence has the mean bound
property and essential radius $O(1/s)$). The proof of Theorem~\ref{thm:main}
therefore goes through almost identically the slope-$s$ ReLU units replacing
the $s$-sharp sigmoid units. In particular, there is a family of single hidden
layer NNs using $O(s\sqrt{n}\log(\lambda sn)$ slope-$s$ ReLU units, which is
not learned by any SQ algorithm using
fewer than $2^{\Omega(n)}/(\lambda s^2)$ queries to $\VSTAT(O(s^2 n))$, when
inputs are drawn i.i.d.\ from a logconcave distribution.

Similarly, we can consider the \emph{$s$-sharp softplus} function $\sigma(x) =
\log(\exp(sx)+1)$. Then Eq.~\eqref{eq:relu} again gives an appropriate
$L^1(\bbR)$ function to which we can apply Lemma~\ref{lem:l1-quasiperiodic} and
therefore follow the proof of Theorem~\ref{thm:main}. For softsign functions
$\sigma(x) = x/(|x|+1)$, we use the affine combination
\[
    \psi(x) = \sigma(x + 1) + \sigma(-x+1)\,.
\]
In the case of softsign functions, this function $\psi$ converges much more
slowly to zero as $|x| \to \infty$ compared to sigmoid units. Hence, in order
to obtain an adequate quasiperiodic function as an affine combination of
$\psi$-units, a much larger number of $\psi$-units is needed: the bound on the
number $m$ of units in this case is polynomial in the Lipschitz parameter
$\lambda$ of the query functions, and a larger polynomial in the input
dimension $n$. The case of other commonly used activation functions, such as ELU (exponential
linear) or LReLU (Leaky ReLU), is similar to those discussed above. 

\section{Experiments}\label{sec:expts}
In the experiments, we show how the errors, $\mathbb{E}(f(x) - y)^2$, change
with respect to the sharpness parameter $s$ and the input dimension $n$ for two
input distributions: 1) multivariate normal distribution, 2) coordinate-wise
independent $\exp(-|x_i|)$,
and 3) uniform in the $l_1$ ball $\{x : \sum_i |x_i| \le n\}$. 

For a given sharpness parameter $s \in \{0.01, 0.02, 0.05, 0.1, 0.2, 0.5, 1, 2\}$, input dimension $d \in \{50, 100, 200\}$ and input distribution, we generate the true function according to Eqn.~\ref{fmS}. There are a total of 50,000 training data points and 1000 test data points.
We then learn the true function with fully-connected neural networks of both ReLU and sigmoid activation functions. The best test error is reported among the following different hyper-parameters.

The number of hidden layers we used is 1, 2, and 4. The number of hidden units per layer varies from $4n$ to $8n$. The training is carried out using SGD with 0.9 momentum, and we enumerate learning rates from 0.1, 0.01 and 0.001 and batch sizes from 64, 128 and 256.

%
%

From Theorem~\ref{thm:main}, learning such functions should become difficult as $s\sqrt{n}$ increases over a threshold.
In Figure~\ref{fig:exp}, we illustrate this phenomenon. Each curve corresponds to a particular input dimension $n$ and each point in the curve corresponds to a particular smoothness parameter $s$. The x-axis is $s\sqrt{n}$ and the y-axis denotes the test errors.
We can see that at roughly $s\sqrt{n} = 5$, the problem becomes hard even empirically.


\begin{figure}
\center
\begin{subfigure}[b]{0.3\textwidth}
\includegraphics[width=\textwidth]{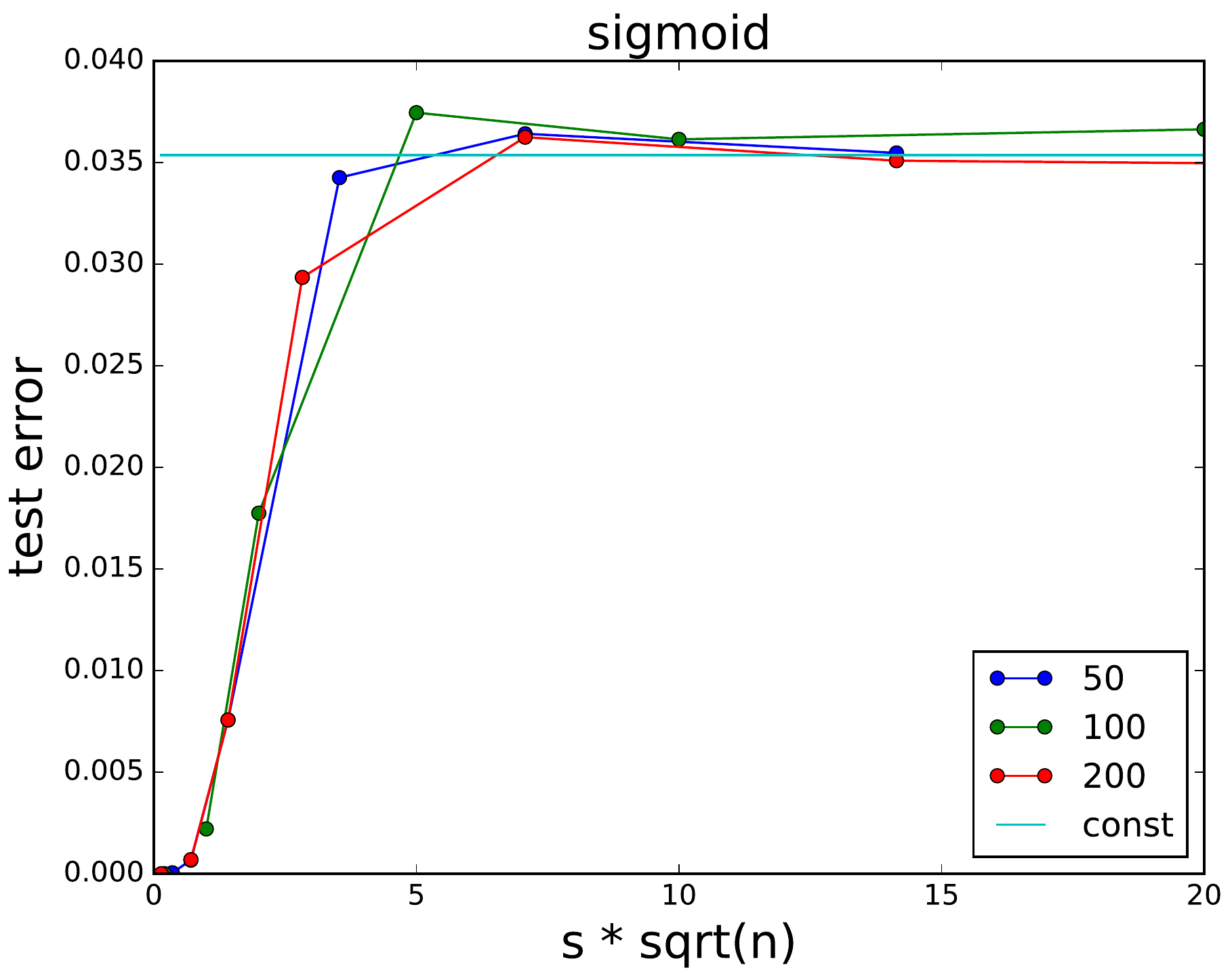}
\caption{normal distribution}
\end{subfigure}%
\begin{subfigure}[b]{0.3\textwidth}
\includegraphics[width=\textwidth]{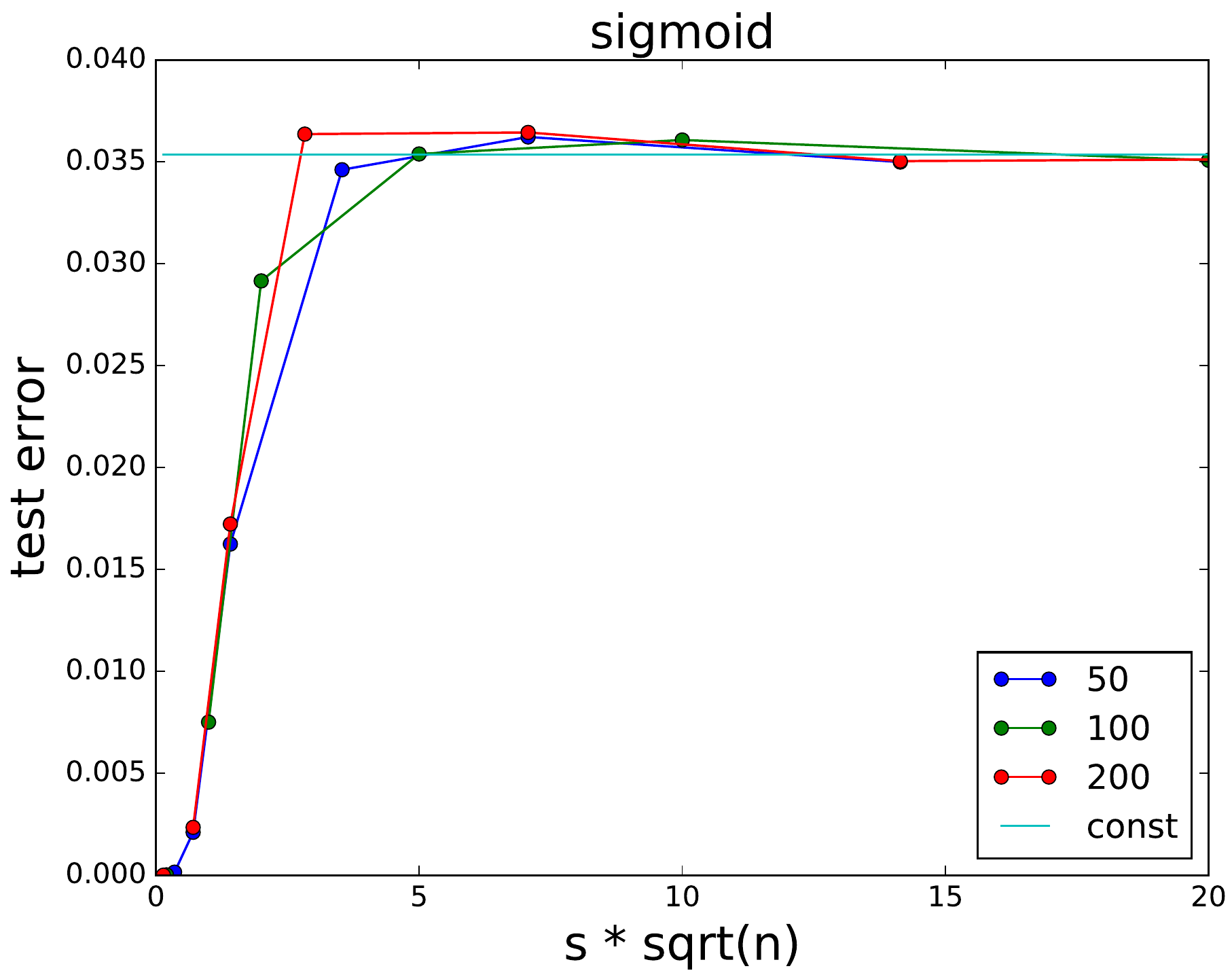}
\caption{$\exp(-|x_i|)$ distribution}
\end{subfigure}
\begin{subfigure}[b]{0.3\textwidth}
\includegraphics[width=\textwidth]{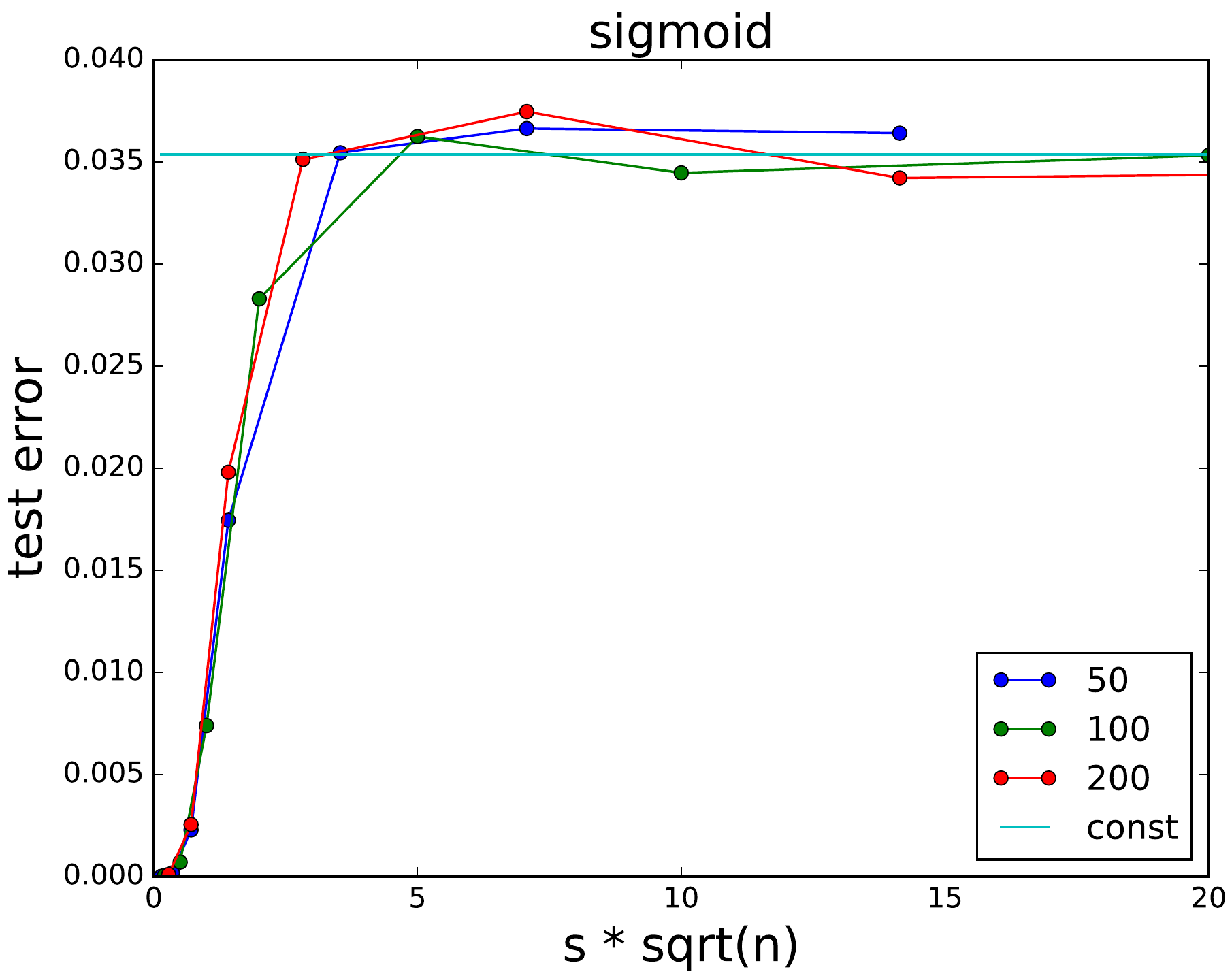}
\caption{uniform $l_1$ ball}
\end{subfigure}
\begin{subfigure}[b]{0.3\textwidth}
\includegraphics[width=\textwidth]{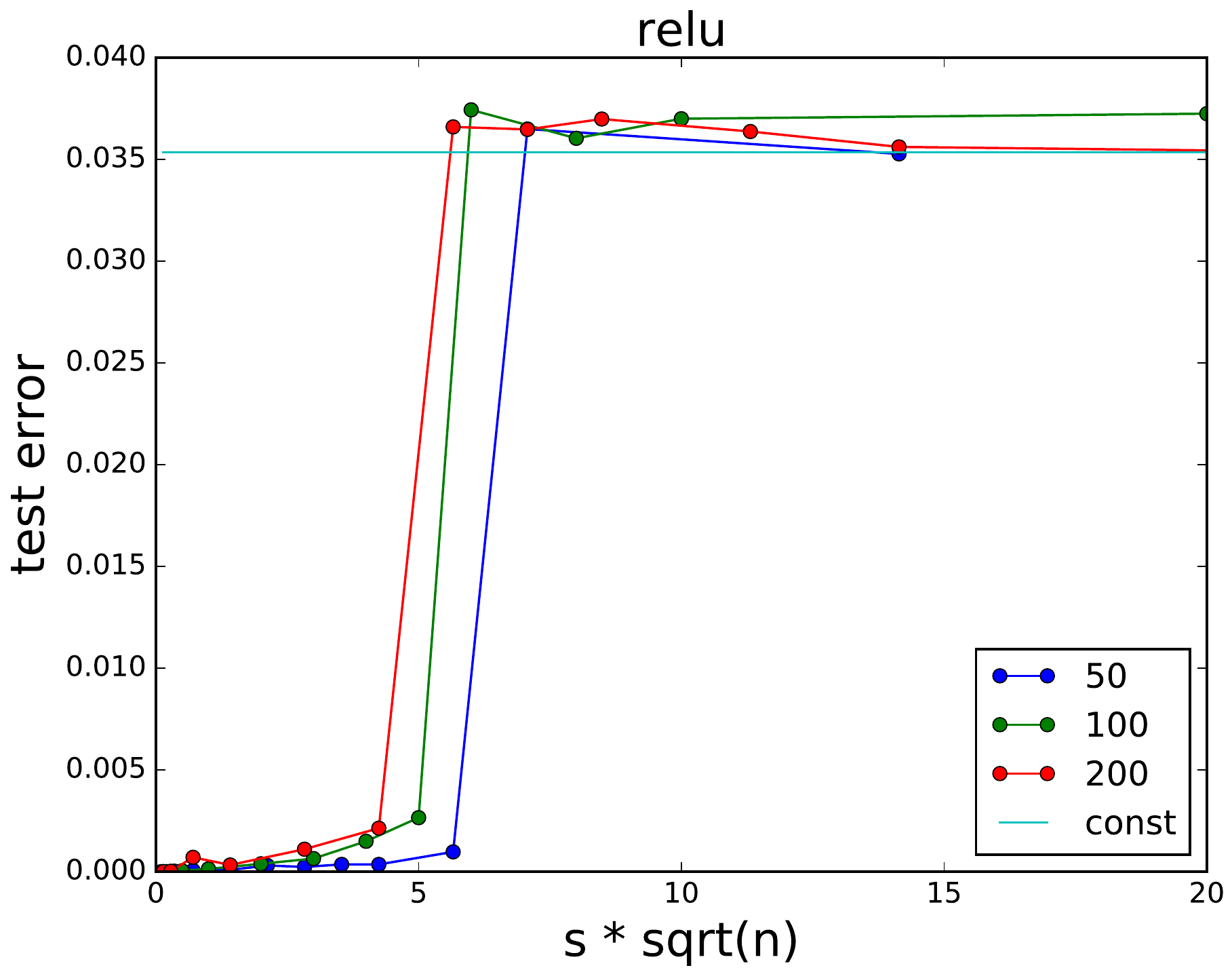}
\caption{normal distribution}
\end{subfigure}%
\begin{subfigure}[b]{0.3\textwidth}
\includegraphics[width=\textwidth]{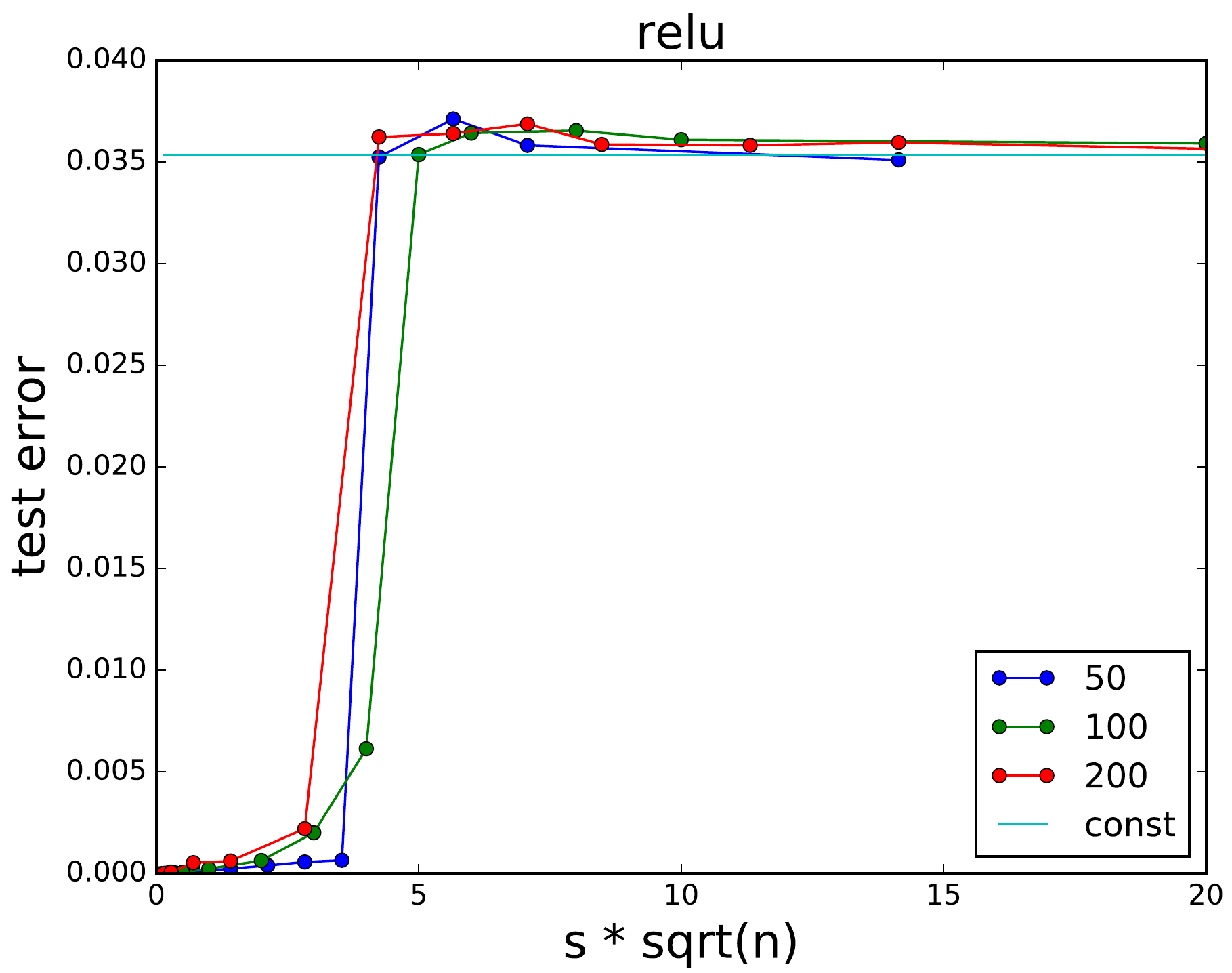}
\caption{$\exp(-|x_i|)$ distribution}
\end{subfigure}%
\begin{subfigure}[b]{0.3\textwidth}
\includegraphics[width=\textwidth]{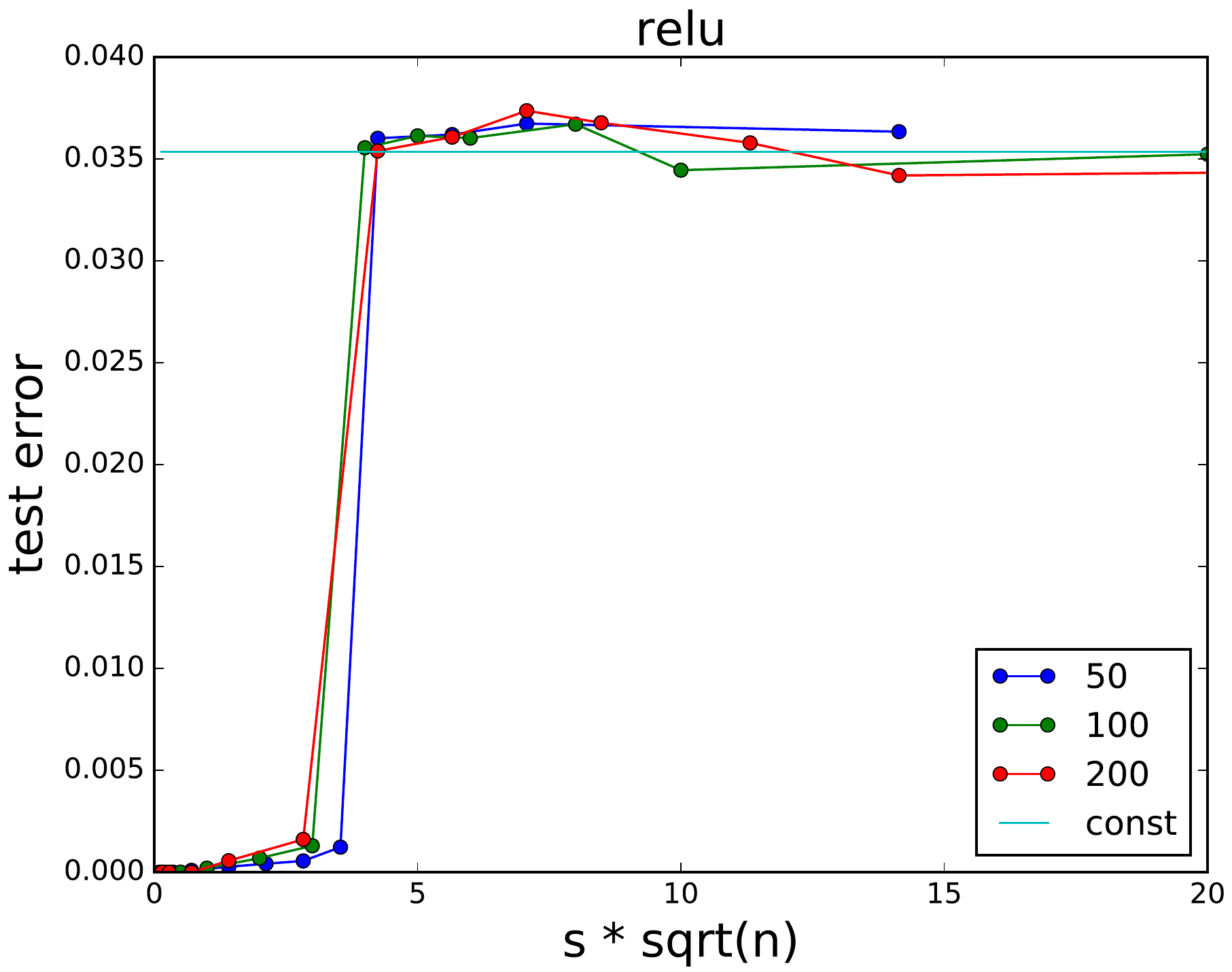}
\caption{uniform $l_1$ ball}
\end{subfigure}%
\caption{Test error vs sharpness times square-root of dimension. Each curve corresponds to a different input dimension $n$. The flat line corresponds to the best error by a constant function.}
\label{fig:exp}
\end{figure}


\section{Complete proof of Theorem~\ref{thm:main}}\label{sec:proofs}
 
\subsection{Statistical dimension with periodic activations}\label{sec:quasiperiodic}

We now prove Lemma~\ref{lem:periodic-sda}.

\begin{lemma}\label{lem:shift-approx}
    Let $\phi:\bbR \to \bbR$ be periodic with period $\theta$. 
    Let $D$ be a probability distribution on $\bbR$ with a unimodal density
    function $f$. Then for any $z \in \bbR$,
    \[
        \left|\E_{x \sim D}(\phi(x) - \phi(x-z))\right| \le
        O\left(\|f\|_{\infty}\right)\int_0^{\theta}|\phi(x)|dx\,.
    \]
\end{lemma}
\begin{proof}
    Since $\phi$ has period $\theta$, by redefining $z$ we may assume without loss
    of generality that $0 < z < \theta$, and that $f$ achieves its mode at $0$.

    By the unimodality of $f$, for any $0 \le x \le \theta$ and any $t \in
    \bbR$ we have
    \[
        |f(x+t) - f(x+t+z)| \le \left\{\begin{matrix} f(t) - f(t+2\theta) & t
            \ge 0 \\
            f(t+2\theta) - f(t) & t \le -2\theta \\
            f(0) - \min\{f(2\theta), f(-2\theta)\} & -2\theta \le t \le
        0\end{matrix}\right.
    \]
    We estimate
    \begin{align*}
        \left|\E_{x \sim D}(\phi(x) - \phi(x-z))\right| &=
        \left|\int_{\bbR}(\phi(x)f(x) - \phi(x-z)f(x))dx\right| \\
        &\le \int_{\bbR}|\phi(x)|\left|f(x) - f(x+z)\right|dx \\
        &= \sum_{k \in \bbZ}\int_0^{\theta}|\phi(x)|\left|f(x+k\theta) - f(x+z
        + k\theta)\right|dx \\
        &\le \left(\left(\sum_{k=0}^{\infty}f(k\theta) -
        f((k+2)\theta)\right)
        + \left(\sum_{k=-\infty}^{0}f(k\theta) -
        f((k-2)\theta)\right)
        + f(0)\right)\int_0^{\theta}|\phi(x)|dx \\
        &= (3f(0) + f(\theta) + f(-\theta))\int_0^{\theta}|\phi(x)|dx \\
        &= O(\|f\|_{\infty})\int_0^{\theta}|\phi(x)|dx\,.
    \end{align*}
\end{proof}

\begin{lemma}\label{lem:quasiperiod-integral}
    Let $\phi: \bbR \to \bbR$ be periodic of period $\theta$, and let $D$ be a
    logconcave distribution on $\bbR$ with variance $\sigma > \theta$. Then
    \[
        \int_0^{\theta}|\phi(x)|dx \le O(\theta)\E_{x \sim
        D}(\phi(x))\,.
    \]
\end{lemma}
\begin{proof}
    Since the quantity $\int_I |\phi(x)|$ is the same for every interval $I$ of
    length $\theta$, we may assume without loss of generality
    that $f$ has its mode at $0$.  We compute
    \begin{align*}
        \int_0^{\theta}|\phi(x)|dx &\le
        \theta\left(f(0) + \sum_{k \in
        \bbZ}f(k\theta)\right)\left(\int_0^\theta|\phi(x)|dx\right) \\
        &= 2\theta\int_0^{\theta}|\phi(x)|f(0)dx +
        \theta\sum_{k \in \bbZ\setminus\{0\}}\int_0^{\theta}|\phi(x)|f(k\theta)
        \\
        &\le 2\theta\int_0^{\theta}|\phi(x)|\cdot O(f(x))dx +
        \theta\int_{\bbR}|\phi(x)|f(x)dx \\
        &= O(\theta)\int_{\bbR}|\phi(x)|f(x)dx = O(\theta)\E_{x \sim
        D}(|\phi(x)|)
    \end{align*}
\end{proof}

\begin{lemma}\label{lem:quasiperiod-approx}
    Let $\phi: \bbR \to [-1,1]$ be $(M,\delta,\theta)$-quasiperiodic, and let
    $\tilde{\phi}: \bbR \to \bbR$ be periodic of period $\theta$ such that
    $|\phi(x)-\tilde{\phi}(x)| < \delta$ for all $|x| \le M$. Suppose $D$ is a
    logconcave distribution on $\bbR$ with mean $\mu$ and variance $\sigma$ such
    that $M > \sigma(1+\log(1/\delta)) + |\mu|$. Then
    \[
        \left|\E_{x \sim D}(|\phi(x)|) - \E_{x \sim
        D}(|\tilde{\phi}(x)|)\right| = O(\delta)
    \]
\end{lemma}
\begin{proof}
    By the tail bound for logconcave distributions, 
    \[
        \left|\E_{x \sim D}(\phi(x)- \tilde{\phi}(x))\right|
        < 2 P_{x\sim D}(|x|
        > M) + \sup_{|x| \le M}(|\phi(x) - \tilde{\phi}(x)|) = O(\delta).
    \]
\end{proof}

\begin{lemma}\label{lem:combined-approx}
    Let $\phi : \bbR \to [-1,1]$ be $(M,\delta,\theta)$-quasiperiodic,
    and let $D$ be a logconcave probability distribution on $\bbR$ with
    variance $\sigma > \theta$ and mean $\mu$. Suppose $z \in \bbR$ is such that $|z| < M -
    \sigma(1+\log(1/\delta)) - |\mu|$. Then 
    \[
        \left|\E_{x \sim D}(\phi(x))
        - \E_{x \sim D}(\phi(x-z))\right| \le
        O\left(\delta + \frac{\theta}{\sigma}\cdot
        \E_{x \sim D}(|\phi(x)|)\}\right)\,.
    \]
\end{lemma}
\begin{proof}
    Let $\tilde{\phi}: \bbR \to \bbR$ be periodic with period $\theta$ such
    that $|\phi(x)-\tilde{\phi}(x)| < \delta$ for all $x \in [-M,M]$. 
    Let $z' \in \bbR$ be such that $|z'| < M- \sigma(1+\log(1/\delta)) - |\mu|$.
    By Lemma~\ref{lem:quasiperiod-approx}, we have
    \begin{equation}\label{eq:tilde-approx}
        \left|\E_{x \sim D}(\phi(x+z') -
        \tilde{\phi}(x+z'))\right| = O(\delta).
    \end{equation}
    Note that the probability density function $f$ of $D$ satisfies $\|f\|_{\infty}
    = O(1/\sigma)$ since $D$ is logconcave. Therefore, using the above
    estimate for both $z'=z$ and $z'=0$, and applying
    Lemmas~\ref{lem:shift-approx} and~\ref{lem:quasiperiod-integral}, we have
    \begin{align*}
        \left|\E_{x \sim D}(\phi(x)) -
        \E_{x \sim D}(\phi(x-z))\right|
        &\le \left|\E_{x \sim D}(\tilde{\phi}(x) -
        \tilde{\phi}(x-z))\right| + O(\delta) \\
        &= O\left(\delta + \frac{1}{\sigma}\cdot \int_0^\theta
        |\tilde{\phi}(x)|dx\right) \\
        &= O\left(\delta + \frac{\theta}{\sigma}\cdot
        \E_{x\sim D}(|\tilde{\phi}(x)|)\right) \\
        &= O\left(\delta + \frac{\theta}{\sigma}\cdot
        \E_{x\sim D}(|\phi(x)|)\right) \\
    \end{align*}
    using Eq.~\eqref{eq:tilde-approx} again for the last estimate.
\end{proof}

The following proposition follows from a straightforward application of the
probabilistic method.

\begin{prop}\label{prop:good-family}
    There exists an absolute constant $c > 0$ such that the following holds.
    For any $n \in \bbN$, there exists a family $\mcF$ of subsets $S \subseteq
    \{1,\ldots, n\}$ such that every $S \in \mcF$ has size $|S| = \lfloor n/2
    \rfloor$, every distinct $S, T \in \mcF$ satisfies $|S\cap T| < (1/2 - c)n$ and
    $|\mcF| \ge 2^{c n}$.
\end{prop}

\begin{proof}[Proof of Lemma~\ref{lem:periodic-sda}]
    Let $X = (x_1,\ldots,x_n) \sim D$.
    Without loss of generality, by shifting the affine functions $g$ we will
    define, we may assume without loss of generality that each $x_i$ has mean zero.
    For every subset $S$ of $\{1,\ldots, n\}$ let 
    \[
    \zeta_S(x_1,\ldots,x_n) = \sum_{i \in S} x_i\,.
    \]
    Let $c > 0$ be the absolute constant and $\mcF$ the family of subsets
    of $\{1,\ldots n\}$ given by
    Proposition~\ref{prop:good-family}.
    Let $\mcC = \{\phi\ \circ \zeta_S : S \in \mcF\}$.
    We show $\eSDA(\mcC, D, \bar{\gamma}) \ge \eps\theta^2 2^{\Omega(n)}$ for some
    $\bar{\gamma} = O(\theta^2/n)$.

        Let $h: [-1,1] \to [-1,1]$ be a $(1/\eps)^2$-Lipschitz function. 
    We write $f_S = h \circ \phi \circ \zeta_S$.
    Let $S, T \in \mcF$, so $|S\setminus T|, |T\setminus S| \ge c n$.
    Let $\mu = \bbE_D(f_S)$ for some, hence any, such
    set $S$.  We claim that
    \begin{equation}\label{eq:pairwise-cov}
        \Cov_D(f_S, f_T) \le O(\max\{\eps,
        \mu\})^2\frac{\theta^2}{n}\,.
    \end{equation}
    This suffices to prove the lemma, as we now show.
    
    First, suppose $h$ is the identity function $h(x) = x$. 
    Observe that since
    $\Var_{x \sim U(0,\theta)}(\phi(x)) = \Omega(1)$, for $\theta = O(\sqrt{n})$
    we also have $\Var_{X \sim D}(f(X)) = \Omega(1)$ for any $f \in \mcC$.
    Then by Eq.~\eqref{eq:pairwise-cov}, for any $S,T \in \mcF$, we have
    \[
        \rho_D(f_S, f_T) \le O\left(\frac{\theta^2}{n}\right)\,.
    \]
    Hence, for any subset $\mcC' \subseteq \mcC$ we have
    \begin{align*}
        \rho_D(\mcC') = 
        \frac{1}{|\mcC'|^2} \sum_{S,T\in\mcC'} \rho(f_S, f_T) &\le
        \frac{1}{|\mcC'|^2}\left(|\mcC'| + \sum_{S\ne T}
        O\left(\frac{\theta^2}{n}\right)\right) \\
    \end{align*}
    This quantity is at most $O(\theta^2/n)$, assuming $|\mcC'| \ge
    n/\theta^2$. In particular, it holds whenever $|\mcC'| \ge |\mcC|/d$ where
    $d = (\theta^2/n)2^{cn}$.
 
    Now let $h = \chi_y^{(\eps)}$ for some $y \in \bbR$, where
    $\chi_y^{(\eps)}$ is the $\eps$-soft indicator function. Setting $\mcC_y =
    \{h \circ f : f \in \mcC\}$ and $\mcC'_y \subseteq \mcC_y$, we have,
    similar to the above, that
    \[
        \Cov_D(\mcC_y') \le O\left(\max\{\eps,
        \mu\}\right)^2\cdot\frac{\theta^2}{n}
    \]
even when $|\mcC'_y| = |\mcC_y|/d$ for some $d = \eps\theta^2 2^{\Omega(n)}$.  This proves that $\eSDA(\mcC,
D, \bar{\gamma}) = \eps\theta^2 2^{\Omega(n)}$.

    It remains to prove Eq.~\eqref{eq:pairwise-cov}.

    Note that since $h$ has Lipschitz constant $(1/\eps)^2$ and $\phi$ is
    $(M,\delta,\theta)$-quasiperiodic, we have that
    $h\circ \phi$ is $(M,\delta',\theta)$-quasiperiodic, where $\delta' =
    \delta/\eps^2 = O(\eps\theta/\sqrt{n})$. Let $\hat{\phi} = h\circ \phi$.
    
    Let $\eps' = \delta' + (\theta/\sqrt{n})\mu$, and for any $U \subseteq \{1,\ldots,
    n\}$, let $m_U = \E_X(f_U(X))$.
    For any $z \in \bbR$ with $|z| < M/2$, we may use
    Lemma~\ref{lem:combined-approx} to estimate 
    \begin{align}\label{eq:fS-exp-estimate}
        \left|\E_X(f_S(X) \mid \zeta_{S\cap T}(X)=z) -
        m_{S\cap T}\right|
        &= \left|\E_X(\hat{\phi}(\zeta_{S\setminus T}(X)+z)) -
        \E_X(\hat{\phi}(\zeta_{S\setminus T}(X)))\right| \\
        &\le O(\eps').
    \end{align}

    By the tail bound for logconcave distributions and the definition of
    of $M$, we have $P_X(|\zeta_{S\cap T}(X)| \ge M/2) < O(\eps')$ for
    adequate choice of the constant hidden in the definition of $M$. Therefore,
    \begin{equation}\label{eq:fS-tail}
        \left|\E_X(f_S(X)) - \E_X\left( f_S(X) \mid \left|\zeta_{S\cap
        T}(X)\right| <
        M/4\right)\right| = O(\eps').
    \end{equation}
    By Eq.~\eqref{eq:fS-exp-estimate} and Eq.~\eqref{eq:fS-tail} we thus have
    $|m_S - m_{S \cap T}| = O(\eps')$. 
    Using Eq.~\eqref{eq:fS-exp-estimate} again, we have for all $|z| <
    M/2$ that
    \[
        \left|\E_X\left(f_S(X) - m_S \mid \zeta_{S\cap
        T}(X)=z\right)\right| \le |m_{S\cap T} - m_{S\cap T}| + O(\eps') =
        O(\eps')\,.
    \]
    Symmetrically, 
    \[
        \left|\E_X\left(f_T(X) - m_T \mid \zeta_{S\cap
        T}(X)=z\right)\right| \le O(\eps')\,.
    \]

    Using the fact that $f_T(X)$ and
    $f_S(X)$ are independent random variables after conditioning on the value
    of $\zeta_{S\cap T}(X)$, we have
    \begin{align*}
        \Cov(f_S, f_T)
        &= \E_X\left(\left(f_S(X)-m_S\right)\left(f_T(X)-m_T\right)\right)\\
        &= \E_{z \sim \zeta_{S\cap T}(X)}\left(\E_X\left(f_S(X)-m_S \mid
        \zeta_{S\cap T}=z\right)\E_X\left(f_T(X) - m_T \mid \zeta_{S\cap
        T}=z\right)\right)
    \end{align*}
    Now using a tail bound of the same flavor as in
    Eq.~\eqref{eq:fS-tail}, we have
    \begin{align*}
        \Cov(f_S, f_T)
        &\le \E_{z \sim \zeta_{S\cap T}(X)}\left(\E_X\left(f_S(X)-m_S \mid
        \zeta_{S\cap T}=z\right)\E_X\left(f_T(X) - m_T \mid \zeta_{S\cap
        T}=z\right) \mid |z| < M/2\right) + O(\eps')^2 \\
        &\le O(\eps')^2.
    \end{align*}
    This proves Eq.~\eqref{eq:pairwise-cov}.
\end{proof}

\subsection{Periodic functions from $L^1$ activations}

Given a function $\psi \in L^1(\bbR)$, we define the
\textbf{$(\psi,\theta)$-periodic function} $F_{\psi,\theta} : \bbR \to \bbR$ by
\[
    F_{\psi,\theta}(x) = \sum_{k \in \bbZ} \psi(x - k\theta).
\]
Note that $F_{\psi,\theta}(x)$ converges absolutely almost everywhere since
$\psi
\in L^1(\bbR)$, and $F_{\psi}$ is indeed periodic with period $\theta$. 

\begin{lemma}\label{lem:uniform-variance-lb}
    Let $\psi \in L^1(\bbR)$ have essential radius $r$, let $\theta = 4r$, let
    $F(x) = F_{\psi,\theta}(x)$, and let $I \subseteq \bbR$ be an interval of
    length $\theta$. Then $\int_I |F(x)|dx \le \|\psi\|_1$. Furthermore, 
    there is a partition $I = I_0 \cup I_1$ of $I$ into
    measurable subsets such that $\int_{I_0} \left| F(x)\right| dx
    \ge (2/3)\|\psi\|_1$ and $\int_{I_1} \left| F(x) \right| dx \le
    (1/6)\|\psi\|_1$.
\end{lemma}
\begin{proof}
    By the periodicity of $F$, we may assume without loss of generality that $I
    = [-r, 3r]$.  By the monotone convergence theorem, we have
    \begin{align*}
        \int_{-r}^r \left| F(x)\right| dx &\ge \int_{-r}^r \left|\psi(x)\right| dx -
        \int_{-r}^r \sum_{0 \ne k \in \bbZ}\left| \psi(x + k\theta)\right| dx \\
        &= \int_{-r}^r \left|\psi(x)\right| - \sum_{0 \ne k \in
        \bbZ}\int_{-r}^r\left| \psi(x +
        k\theta)\right| dx \\
        &\ge \int_{-r}^r \left|\psi(x)\right| - \int_{\bbR\setminus
        [-r,3r]}\left|\psi(x)\right| dx.
    \end{align*}
    By the definition of $r$, we therefore have $\int_{-r}^r \left| F(x)\right|
    dx \ge (2/3)\|\psi\|_1$. Similarly
    \[
        \int_{r}^{3r} \left| F(x)\right| dx\le
        \int_{\bbR\setminus[-r,r]}\left|\psi(x)\right| dx \le (1/6)\|\psi\|_1
    \]
\end{proof}

For any $m > 0$, we define the \textbf{truncated $(\psi,\theta)$-periodic function}
\[
    F^{(m)}_{\psi,\theta}(x) = \sum_{\substack{k \in \bbZ \\ |k| \le m}} \psi(x
    - k\theta).
\]

\begin{lemma}\label{lem:Fpsi-sup}
    Let $\psi \in L^1(\bbR)$ have the mean bound property. Then letting either
    $g = F_{\psi,\theta}$ or $g = F^{(m)}_{\psi,\theta}$ for some $m \in \bbN$,
    we have $\sup_{x \in \bbR} |g(x)| = O(\|\psi\|_1/\theta)$.
\end{lemma}
\begin{proof}
    We compute
    \begin{align*}
        |g(x)| &\le \sum_{k \in \bbZ} |\psi(x + k\theta)| \\
        &\le \sum_{k \in \bbZ} 
        O\left(\frac{1}{\theta}\int_{x+(k-1/2)\theta}^{x+(k+1/2)\theta}
        |\psi(y)|dy\right) \\
        &= O\left(\frac{1}{\theta}\int_\bbR |\psi(y)|dy\right).
    \end{align*}
\end{proof}

Despite its name, the truncated $\psi$-periodic function is not in general
periodic. Nevertheless, it approximates $F_{\psi,\theta}$.

\begin{lemma}\label{lem:truncated-approx}
    Let $\psi: \bbR \to \bbR$ have the mean bound property, and let $\theta >
    0$. Then for all $x \in \bbR$ with $|x| \le m\theta/2$, we have
    \[
        \left|F_{\psi,\theta}^{(m)}(x) - F_{\psi,\theta}(x)\right| \le 
        \frac{1}{\theta}\int_{m\theta/2}^\infty (|\psi(x)| + |\psi(-x)|)dx\,.
    \]
\end{lemma}
\begin{proof}
    Indeed, we have
    \begin{align*}
        \left|F_{\psi,\theta}^{(m)}(x) - F_{\psi,\theta}(x)\right| &= \left|
        \sum_{|k| > m} \psi(x + k\theta) \right| \\
        &\le \sum_{|k| > m} \left| \psi(x + k\theta) \right| \\
        &\le \sum_{|k| > m} O\left(\frac{1}{\theta}\int_{x+(k-1/2)\theta}^{x+(k+1/2)\theta}
        |\psi(y)|dy\right) \\
        &\le \frac{1}{\theta}O\left(\int_{m\theta/2}^{\infty} |\psi(y)|dy +
        \int_{-\infty}^{-m\theta/2} |\psi(y)|dy \right)
    \end{align*}
\end{proof}

\begin{proof}[Proof of Lemma~\ref{lem:l1-quasiperiodic}]
    By Lemma~\ref{lem:truncated-approx}, the function $F^{(m)}_{\psi,\theta}$
    is $(M,\delta,\theta)$-quasiperiodic for appropriate choice of $m$. Hence,
    by Lemma~\ref{lem:uniform-variance-lb}, taking $\theta=4r$, the function
    $F^{(m)}_{\psi,\theta}$ also has the desired variance. Finally, by
    Lemma~\ref{lem:Fpsi-sup}, we may rescale $F^{(m)}_{\psi,\theta}$ by a
    constant factor to ensure that its range is in $[-1,1]$, preserving the
    variance (up to constant factors) and quasiperiodicity (for appropriate
    choice of $m$).
\end{proof}

\subsection{Proof of Main Theorem and Corollary}

We now give the full proof of Theorem~\ref{thm:main}, sketched previously in
Section~\ref{sec:periodic}. First, we prove a small lemma that will be useful
for the condition number guarantee.

\begin{lemma}\label{lem:noisy-esda}
    Let $D$ be a distribution on a domain $X$ and let $\mcC$ and $\mcC'$ be
    families of functions $f: X \to \bbR$ with variance $\Var_D(f) = 1$. 
    Suppose that for some $\delta > 0$ there is a bijection $\rho: \mcC \to \mcC'$ such that $\|\rho(f) -
    f\|_{\infty} < \delta$ for all $f \in \mcC$, and there is $d \in \bbN$ and
    $\eps, \bar{\gamma} > 0$ such that 
    $\eSDA(\mcC, D, \bar{\gamma}) = d$. Suppose further that the functions of
    $\mcC$ are identically distributed over $D$, as are the functions of
    $\mcC'$.
    Then if $\delta < \eps$, we have
    $\eSDA(\mcC', D, \bar{\gamma}') = d$,
    where $\bar{\gamma}' = (1 + O(\delta/(\eps^3\bar{\gamma})))\bar{\gamma}$.
\end{lemma}
\begin{proof}
    Let $h : \bbR \to \bbR$ be $(1/\eps)^2$-Lipschitz, and let $f,g \in
    \mcC$. Then
    \begin{align*}
        \left|\Cov_D(h \circ \rho(f), h \circ \rho(g))\right| &=
        \left|\Cov_D((h \circ
        \rho(f) - h \circ f) + h \circ f, (h \circ \rho(g) -h 
        \circ g) + h \circ g)\right| \\
        &\le \Cov_D(h \circ f, h \circ g) + O(\delta/\eps^2 +
        \delta^2/\eps^4)\,.
    \end{align*}
    The lemma follows by setting $h$ to be either the identity function or a
    soft indicator $\chi_y^{(\eps)}$, and averaging over all pairs $f,g \in
    \mcC$.
\end{proof}

\begin{proof}[Proof of Theorem~\ref{thm:main}]
    The sigmoid function $\sigma$ with sharpness $s$ is not even in
    $L^1(\bbR)$, so it is unsuitable as the function $\psi$ of
    Lemma~\ref{lem:l1-quasiperiodic}. Instead, we define $\psi$ to be an affine
    combination of $\sigma$ gates, namely
    \[
        \psi(x) = \sigma\left(\frac{1}{s} + x\right) + \sigma\left(\frac{1}{s}
        - x\right) - 1.
    \]
    Then $\|\psi\|_1 = 2/s$, and 
    \[
    \int_{-r}^r|\psi(x)| \ge \frac{5}{6}\norm{\psi}_1,
    \]
    for some $r=\Theta(1/s)$ and therefore this is a bound on the essential
    radius of $\psi$.
    Furthermore, since $\psi$ is monotonic for sufficiently large positive or
    negative inputs, $\psi$ has the mean bound property. 
    Thus, $\psi$ satisfies the hypotheses of Lemma~\ref{lem:l1-quasiperiodic}.

    Let $\theta = 4r$ and let $\bar{\gamma} = O(\theta^2/n)$ be as given by the statement of
    Lemma~\ref{lem:periodic-sda}. Let $\eps = \bar{\gamma}/(2\lambda)$, and let 
    $M = O(\sqrt{n}\log(n/(\eps\theta))$ and $\delta = \Omega(\eps^3\theta/\sqrt{n})$
    be as given by the statement of Lemma~\ref{lem:periodic-sda}. By
    Lemma~\ref{lem:l1-quasiperiodic}, there is
    $m \in \bbN$ and 
    functions $h: \bbR^m \to \bbR$ and $g: \bbR \to \bbR^m$ such that
    $\phi = h \circ \psi \circ g$ is $(M,\delta,\theta)$-quasiperiodic and
    satisfies the hypotheses of Lemma~\ref{lem:periodic-sda}.
    Therefore, we have a family $\mcC_0$ of affine functions $f: \bbR^n \to \bbR$ such
    that for $\mcC = \{\phi \circ f : f \in \mcC_0\}$ satisfies
    $\eSDA(\mcC, D, \bar{\gamma}) \ge 2^{\Omega(n)}\eps\theta^2$.
    Furthermore, the functions in $\mcC$ are identically distributed.
    Therefore, the functions in $\mcC$ satisfy the hypothesis of
    Theorem~\ref{thm:sda}, giving the query complexity lower bound.
    
    Note that the functions in $\mcC$ are represented
    by single-layer neural networks, the composition of any $f \in \mcC_0$ with
    $g$ is again an affine function.

    It remains to estimate the number $m$ of $\psi$-units used to represent the
    functions in $\mcC$, which is half the number of $\sigma$ activation units used.

    By Lemma~\ref{lem:l1-quasiperiodic}, we may take $m = O(\max\{m_1,
    s\sqrt{n}\log(n/(\eps\theta))\})$, where $m_1$ satisfies
    \[
        \int_{m_1}^\infty (|\psi(x)| + |\psi(-x)|)dx <
        O\left(\frac{\bar{\gamma}^3}{\lambda^3 s\sqrt{n}}\right) =
        O\left(\frac{1}{\poly(\lambda, s, n)}\right)
    \]
    We note that
    \begin{align*}
        \int_{m_1}^\infty (|\psi(x)| + |\psi(-x)|)dx &= 2\int_{m_1}^\infty \frac{e^2-1}{(1+e^{1+sx})(1+e^{1-sx})}\,dx\\ 
        &= O(1) \frac{e^{-sm_1}}{s}.
        \end{align*}
    Therefore, some $m_1 = O(\log(\lambda s n)/s)$ suffices and this implies that the
    number of $\psi$-units used is $m = O(s\sqrt{n}\log (\lambda s n))$.

    Up to this point, the  weight matrices of the NNs produced by the direct application
    of Lemmas~\ref{lem:l1-quasiperiodic} and~\ref{lem:periodic-sda}
    have rank $1$. In order to obtain the condition number guarantee of
    Theorem~\ref{thm:main}, it is therefore necessary to modify the weight
    matrices, which we accomplish by adding Gaussian noise with
    variance
    $1/\poly(\lambda,s,n)$ in each coordinate. We now sketch this analysis.

    The functions in the family $\mcC$
    have the form
    $\phi \circ h$ for some quasiperiodic
    function $\phi: \bbR \to \bbR$ constructed from an affine combination of
    the activation functions, and affine function $h$. In particular, the weight matrix of the
    corresponding one-layer NN has columns equal to $0$
    whenever the column index is not in $S$, and equal to some fixed vector $v$
    with bounded entries whenever the column index is in $S$.  (The weight
    matrix is thus rank $1$.) Furthermore, for every $f,g \in \mcC$ there is a
    column permutation $\pi_{fg}$ transforming the weight matrix for the NN
    computing $f$ to the weight matrix for $g$.

    Fix $f \in \mcC$, let $W_f$ be the $m\times n$ weight matrix for $f$, and
    let $N$ be an $m \times n$ matrix with entries drawn independently from
    $\mcN(0,\delta)$ for some $\delta \in \bbR$ to be specified. With probability
    $1$ the matrix $W_f + N$ has condition number $O(m/\delta)$, and with
    probability $\Omega(1)$ the matrix $N$ has all entries at most
    $O(\sqrt{\delta\log(mn)})$ in absolute value. Let $\tilde{f}$ be the function
    computed by the NN obtained by replacing $W_f$ with $W_f + N$. Since the
    $m$ activation units are $O(s)$-Lipschitz, we have $|\tilde{f}(x) - f(x)|
    \le O(msn\sqrt{\delta\log(mn)})$. For $g \in \mcC$, let $\rho(g)$ be the
    function obtained by replacing $W_g$ with $\pi_{fg}(W_f + N)$ as the weight
    matrix of the NN computing the function, and let $\mcC' = \{\rho(g) : g \in
    \mcC\}$. The functions in $\mcC'$ are identically distributed over the
    input distribution $D$, since the functions in $\mcC$ are.
    Hence, for some $\delta = \Omega(1/\poly(\lambda,n,s))$,
    Lemma~\ref{lem:noisy-esda} gives that $\eSDA(\mcC, D, \bar{\gamma}) \ge d$
    implies $\eSDA(\mcC', D, O(\bar{\gamma}) \ge d$. By Theorem~\ref{thm:sda},
    we therefore get the same statistical query complexity guarantee for $\mcC$
    as for $\mcC'$, up to a constant factor change in the query tolerance.
\end{proof}

We conclude with a corollary showing that the Lipschitz assumption on the
statistical queries can be omitted when the function outputs are represented
with finite precision.

\begin{corollary}
\label{cor:finite-precision} Suppose the family $\mcC$ of functions from
    Theorem~\ref{thm:main} have outputs rounded to some uniformly-spaced finite
    set $Y \subseteq [-1,1]$, and let $n$, $D$, $s$, and $t$ be as in the
    statement of the theorem. Let $A$
be a randomized SQ algorithm learning $\mcC$ over $D$
    to regression error less than $\Omega(1) - \sqrt{1/t}$ with probability at least $1/2$. 
    Then if $A$ uses arbitrary queries to $\VSTAT(t)$,
    it requires at least
    $2^{\Omega(n)}/(|Y| s^2)$ queries.
\end{corollary}
\begin{proof}
    In fact, since Theorem~\ref{thm:main} follows from a statistical dimension
    bound via Theorem~\ref{thm:sda},
    we can in any case relax the assumption that query functions $h: X\times
    \bbR \to [0,1]$ are $\lambda$-Lipschitz to the assumption used in
    Theorem~\ref{thm:sda}: namely, that the query functions are
    $\lambda$-Lipschitz after fixing $x \in X$. We now consider query functions
    $h: X \times Y \to [0,1]$, where $Y \subseteq [-1,1]$ is a uniformly-spaced
    finite set. But every such function $h$ is $2/|Y|$-Lipschitz at every fixed
    $x \in X$.
\end{proof}

The bound on the number of sigmoid gates used is $O(s\sqrt{n}\log(|Y|sn))$,
similar to the bound in Theorem~\ref{thm:main}. On the other hand, 
the weight matrices used in the neural networks defining the family $\mcC$ of
Corollary~\ref{cor:finite-precision} have rank $1$, assuming they are
represented with the same precision as the function outputs.

\section{Query Complexity via Statistical Dimension}\label{sec:sda-proof}

We now give the proof of Theorem~\ref{thm:sda}, the query complexity bound that
follows from a bound on statistical dimension.

\begin{prop}\label{prop:query-est}
    Let $D$ be a distribution on a domain $X$, let $f:X \to \bbR$, let $h : X
    \times \bbR \to [0,1]$ be $\lambda$-Lipschitz. Then for any $\eps > 0$
    \[
        \left|\E_{x \sim D}(h(x,f(x))) - \int_{\bbR}\E_{x \sim
        D}(h(x,y)\chi_y(f(x)))dy\right|
        \le (5/3)\lambda\eps.
    \]
\end{prop}
\begin{proof}
    Since $\|\chi_y\|_1 = 1$, we have
    \begin{align*}
        \left|\E_{x \sim D}(h(x, f(x))) - \int_{\bbR}\E_{x \sim
        D}(h(x,y)\chi_y(f(x)))dy\right|
        &=
        \left|\E_{x \sim D}(\int_{\bbR}(h(x,
        f(x))-h(x,y))\chi_y(f(x))dy)\right|\,.
    \end{align*}
    We compute
    \begin{align*}
        \int_\bbR\left|h(x,f(x))-h(x,y)\right|\chi_y(f(x))dy
        &\le \int_{f(x)-\eps}^{f(x)+\eps}\lambda|f(x)-y|(1/\eps +
        (1/\eps)^2|f(x)-y|)dy = (5/3)\lambda\eps\,.
    \end{align*}
\end{proof}

\begin{proof}[Proof of Theorem~\ref{thm:sda}]
    Let $\mcC'\subseteq \mcC$ have size greater than $|\mcC|/d$, and let
    $\mcC'' = \{(g-\E_D(g))/\sqrt{\Var(g)} : g \in \mcC'\}$. We have
    $\rho_D(\mcC'') = \rho_D(\mcC') \le \bar{\gamma}$ by assumption. Let $f :
    X \to \bbR$ satisfy $\E_D(f) = 0$ and $\Var(f) = 1$. Then by
    Cauchy-Schwarz,
    \begin{align*}
        \sum_{g \in \mcC''} \rho_D(f, g) &= \E_{x \sim D}\left(f(x)\sum_{g
        \in \mcC''}g(x)\right)
        \le \|f\|_{2,D}\left\|\sum_{g \in \mcC''} g\right\|_{2,D} \\
        &\le \sqrt{\sum_{g,h \in \mcC''}\E_{x \sim D}\left(g(x)h(x)\right)}
        \le |\mcC''|\sqrt{\bar{\gamma}}.
    \end{align*}
    Hence, by Markov's inequality, given a random $g \in \mcC'$, the
    probability that $\rho_D(f, g) \ge 2\sqrt{\bar{\gamma}}$ is at most $1/2$.
    Thus, in order to learn $f$ with regression error less than $1 -
    2\sqrt{\bar{\gamma}}$ with probability greater than $1/2$, the statistical
    algorithm must first rule out all but at most $|\mcC|/d$ of the functions
    in $\mcC$.

    Let $h: X \times \bbR \to [0,1]$ be a $\lambda$-Lipschitz query function.
    Let $\mu(y) = \E_D(\chi_y \circ f)$ for some, hence any, $f \in \mcC$, and
    let $T = \{y \in \bbR : \mu(y) \ge \eps\}$. Since $\eps < \bar{\gamma}/2
    \le 1/2$, the support of $\mu$ is contained in $[-1/2,3/2]$. Hence,
    \[
        \int_{\bbR\setminus T}\bbE_{x \sim D}(h(x,y)\chi(f(x)-y))dy \le
        \int_{\bbR\setminus T} \mu(y)dy < \int_{-1/2}^{3/2} \eps < \bar{\gamma}.
    \]
    By Proposition~\ref{prop:query-est}, we therefore have
    \begin{equation}\label{eq:query-est}
        \left|\E_{x \sim D}(h(x,f(x))) - \int_T\E_{x \sim
        D}((h(x,y)\chi(f(x)-y))dy\right|
        < 2\bar{\gamma}.
    \end{equation}

    We define
    \[
        v = \int_T \E_{x \sim D}(h(x,y))\mu(y)dy\,.
    \]
    We will examine the situation in which the oracle responds to query $h$ with
    value $v$. Let $\mcC' \subseteq \mcC$, and let $h_y(x) = h(x,y)$. We estimate using Cauchy-Schwarz,
    \begin{align*}
        \int_T\left\langle h_y, \sum_{f \in \mcC'}
        \chi_y \circ f -\mu(y) \right\rangle_{x \sim D} dy
        &\le \int_T \|h_y\|_{2,D}
        \left\|\sum_{f \in \mcC'} \chi_y \circ f - \mu(y)
        \right\|_{2,D}dy \\
        &=
        \int_T \|h_y\|_{2,D}\sqrt{\sum_{f, g\in \mcC'}
        \langle \chi_y \circ f - \mu(y),
        \chi_y \circ g - \mu(y)\rangle_D}dy \\
        &=
        \left(\int_T 
        \|h_y\|_{2,D}\sqrt{\Cov_D(\mcC_y')}dy\right)|\mcC'|.
    \end{align*}

    Suppose that $|\mcC'| > |\mcC|/d$. Then
    \begin{align*}
        \int_T \|h_y\|_{2,D}\sqrt{\Cov_D(\mcC_y')}dy
        &\le
        \int_T 
        \|h_y\|_{2,D}\sqrt{\bar{\gamma}}\max\{\eps,\mu(y)\}dy \\
        &=
        \sqrt{\bar{\gamma}}\int_T
        \sqrt{|h_y\|_{1,D}}\mu(y) dy \\
        &\le
        \sqrt{\bar{\gamma}v}
    \end{align*}
    by Jensen's inequality.

    On the other hand, by Eq.~\eqref{eq:query-est}, we have
    \begin{align*}
        \int_T \cdot\left\langle h_y, \sum_{f \in \mcC'}
        \chi_y \circ f -\mu_y \right\rangle_D dy
        &= \sum_{f \in \mcC'} \int_T\left(\E_{x \sim D}\left(
        h(x,y)\chi_y(f(x))\right) -  \E_{x \sim D}
        \left(h(x,y)\mu(y)\right)\right)dy \\
        &\ge \sum_{f \in \mcC'} \left(\E_{x \sim D}(h(x, f(x)))
        - v - 2\bar{\gamma}\right) \\
    \end{align*}

    Hence, for every subset $\mcC' \subseteq \mcC$ of size $|\mcC'| \ge
    |\mcC|/d$, we have
    \[
        \sum_{f \in \mcC'} \left(\E_{x \sim D}(h(x, f(x)))
        - v\right) \le |\mcC'|(\sqrt{\bar{\gamma}v} + 2\bar{\gamma})
    \]
    Let $t \in \bbN$ and for $f \in \mcC$ let $p_f = \E_{x \sim D}(h(x, f(x)))$.
    \[
        \mcC' = \left\{ f \in \mcC : p_f - v > \max\{\frac{1}{t}, \sqrt{\frac{p_f(1-p_f)}{t}}\right\}\,.
    \]
    It follows that either
    $|\mcC'| \le |\mcC|/d$, or we have
    (cf.~\cite[Lemma 3.5]{feldman2013statistical}) 
    \[
        \max\left\{\frac{1}{t}, \sqrt{\frac{v(1-v)}{3t}}\right\} <
        \sqrt{\bar{\gamma}v} + 2\bar{\gamma},
    \]
whence $t > \Omega(1/\bar{\gamma})$.  Hence, no query strategy can rule out
more than $2|\mcC|/d$ functions per query to $\VSTAT(c/\bar{\gamma})$ for some
constant $c$. Hence, any statistical algorithm using queries to
$\VSTAT(c/\bar{\gamma})$ requires at least $\Omega(d)$ queries to learn $\mcC$.
\end{proof}

From the theorem, we see that a lower bound on the statistical dimension of a
distributional problem implies a lower bound on the complexity of any
statistical query algorithm for the problem.

\subsubsection*{Acknowledgments}

The authors are grateful to Vitaly Feldman for discussions about
statistical query lower bounds, and for suggestions
that simplified the presentation of our results.

\bibliographystyle{plain}
\bibliography{NN_bib,planted}

\end{document}